\documentclass{article}

\usepackage[accepted]{icml_2021_like}

\usepackage{natbib}
\usepackage{tikz}
\usepackage{subfigure}

\usepackage{hyperref}
\usepackage{amsfonts}
\usepackage{mathtools}
\usepackage{enumitem}
\usepackage{microtype}

\usetikzlibrary{bayesnet}
\usepackage[overload]{textcase}
\usepackage[utf8]{inputenc}
\usepackage{multirow}
\usepackage{amsthm}
\usepackage{xspace}

\newtheorem{theorem}{Theorem}[section]
\newtheorem{lemma}[theorem]{Lemma}

\newtheorem{proposition}[theorem]{Proposition}
\newtheorem{corollary}[theorem]{Corollary}
\newtheorem{definition}[theorem]{Definition}

\newcommand{\para}[1]{\medskip \noindent {\bf #1}}

\newcommand{\bz}{\boldsymbol z}
\newcommand{\bZ}{\boldsymbol Z}
\newcommand{\be}{\boldsymbol e}

\newcommand{\bY}{\boldsymbol Y}
\newcommand{\bx}{\boldsymbol x}

\newcommand{\bX}{\boldsymbol X}

\newcommand{\bD}{\boldsymbol D}
\newcommand{\bs}{\boldsymbol s}

\newcommand{\bzero}{\boldsymbol 0}
\newcommand{\by}{\boldsymbol y}

\newcommand{\bt}{\boldsymbol t}

\newcommand{\btheta}{\boldsymbol \theta}
\newcommand{\bmu}{\boldsymbol \mu}
\newcommand{\bSigma}{\boldsymbol \Sigma}

\newcommand{\norm}[1]{\Vert #1 \Vert}

\newcommand{\cov}{\operatorname{Cov}}

\newcommand{\diff}{\,\mathrm{d}}

\begin{document}

\twocolumn[
\icmltitle{Differentially Private Bayesian Inference for  Generalized Linear Models}

\icmlsetsymbol{equal}{*}

\begin{icmlauthorlist}
\icmlauthor{Tejas Kulkarni}{aalto}
\icmlauthor{Joonas Jälkö}{aalto}
\icmlauthor{Antti Koskela}{Helsinki}
\icmlauthor{Samuel Kaski}{aalto,Manchester}
\icmlauthor{Antti Honkela}{Helsinki}

\end{icmlauthorlist}

\icmlaffiliation{aalto}{Aalto University, Finland}
\icmlaffiliation{Helsinki}{University of Helsinki, Finland}
\icmlaffiliation{Manchester}{University of Manchester, United Kingdom}

\icmlcorrespondingauthor{Tejas Kulkarni}{tejasvijaykulkarni@gmail.com}
\icmlcorrespondingauthor{Joonas Jälkö}{joonas.jalko@aalto.fi}

\icmlkeywords{Bayesian ML, Differential Privacy}

\vskip 0.3in
]

\printAffiliationsAndNotice{}  

\begin{abstract}
Generalized linear models (GLMs) such as logistic regression are among the most widely used arms in data analyst’s repertoire and often used on sensitive datasets. A large body of prior works that investigate GLMs under differential privacy (DP) constraints provide only private point estimates of the regression coefficients, and are not able to quantify parameter uncertainty.

In this work, with logistic and Poisson regression as running examples, we introduce a generic noise-aware DP Bayesian inference method for a GLM at hand, given a noisy sum of summary statistics. Quantifying uncertainty allows us to determine which of the regression coefficients are statistically significantly different from zero. We provide a previously unknown tight privacy analysis and experimentally demonstrate that the posteriors obtained from our model, while adhering to strong privacy guarantees, are close to the non-private posteriors.
\end{abstract}

\section{Introduction}

Differential privacy (DP)~\cite{DworkMNS06} provides a strong framework for protecting the privacy of data subjects against privacy violations via models trained on their personal data. DP protection requires injecting noise to the learning process. Bayesian inference is a natural complement to DP, because it seeks to quantify the impact of noise to inference result in terms of quantifying the uncertainty of the result. In our work we seek to develop a Bayesian method to perform inference under DP and quantify the uncertainty caused by the injected noise for the widely used class of regression models, generalised linear models (GLMs).
This method allows statistical inference on the regression coefficients, such as determining which coefficients can be confidently inferred to be different from zero.

Using Bayesian inference to counter the noise injected to ensure DP was first proposed by \citet{WM:10}. The process is fairly straightforward for models where the joint distribution $\Pr[\bD, \btheta, \bZ]$ over the data $\bD$, all parameters $\btheta$ and possible latent variables $\bZ$ of interest is specified as part of the model. This was demonstrated by \citet{BS:18}, who presented an efficient inference method for exponential family models of this form.

The problem becomes significantly more difficult for discriminative models such as regression models, where the model does not specify a distribution over the input data $\bX$ but only target outputs $\bY$ and parameters $\Pr[\bY, \btheta \mid \bX]$. This is because we also consider input $\bX$ as private information, and thus we cannot observe it directly. On the other hand, the model does not specify a natural prior for $\bX$ either. In order to solve this issue, we need to augment the model to include a prior for $\bX$. This was first demonstrated by \citet{BS:19}, who developed a method for linear regression by placing a Gaussian prior on $\bX$. The method relies on the ability to express the model via a sufficient statistic of fixed size, and hence only applies to linear regression.

In our work we extend this sufficient statistics based noise-aware DP Bayesian inference framework for a much broader class of regression models, GLMs. These include many of the most widely used statistical models, such as logistic regression and Poisson regression. As GLMs typically have no sufficient statistics, this is achieved by approximating the joint distribution of the inputs and outputs under the GLM using a finite number of moments as approximate sufficient statistics and fitting the model parameters to match these moments.

\para{Related work.} Linear models have received a huge amount of attention under DP since its proposal. The techniques for point estimates for regression parameters fall to five main categories, a) \emph{objective perturbation}~\citep{CMS:09,KM:11,ZZZXYW:12,INSTTW:19}, b) \emph{output perturbation}~\citep{WKCJN:16,ZKWL:17}, c) \emph{gradient perturbation}~\citep{BST:14,ACGMMTZ:16}, d) \emph{subsample and aggregate}~\citep{Smith:08,DS:10,Barrientos:19}, and e) \emph{sufficient statistic perturbation (SSP)}~\citep{MM:09,VS:09,sheffet:19,WY:18}. Other representative works that specifically study \emph{generalized linear models} (GLMs) more generally under various DP models include \citep{KST:12,JT:14,PKLSYHZ:18,WCX:19,WZGX:19}.
 Only a handful of these works, e.g.\  \citep{sheffet:19,Barrientos:19}, quantify the uncertainty in the model parameter estimates via frequentist tools such as confidence intervals and hypothesis testing. 

\par 
To quantify the uncertainty beyond the frequentist tools, many Bayesian inference techniques under DP setting have been proposed, starting from the seminal work of \citet{WM:10}. The field can be roughly clustered into three broad categories:
\begin{enumerate} \vspace{-0.2cm}
    \item \emph{sufficient statistics perturbation} based inference ~\citep{FGJWC:16,ZRD:16,HDNDK:18,BS:18,BS:19,park2020variational}
     \vspace{-0.2cm}
    \item \emph{gradient perturbation} based Markov chain Monte Carlo (MCMC)~\citep{WYFS:2015,ZRD:16,LCLC:19} and variational inference (VI)~\citep{JDH:17}
     \vspace{-0.2cm}
    \item \emph{DP posterior sampling}~\citep{DNMAR:2018,FGJWC:16,ZRD:16,HJDH:19,Yildirim2019}. 
\end{enumerate}

Categories 2 and 3 aim to provide a general-purpose solution for differentially private Bayesian inference. However, the output of these mechanisms is an approximation of the posterior distribution where the impact of added uncertainty from privacy is not quantified. Category 1 is most closely related to our work. In these works, posterior distribution of the model parameters is captured through perturbed sufficient statistics, but also here many approaches fail to quantify the impact of the added uncertainty from privacy.   

Among the several approaches proposed, inference techniques based on \emph{sufficient statistics perturbation} stand out due to their computational efficiency and accuracy~\cite{WY:18}. Unlike DP variants of general purpose MCMC methods, the privacy cost of training in a sufficient statistics based model is typically invariant to the number of iterations/posterior samples. We only pay once to perturb the sum of sufficient statistics and then rely on the post-processing property of DP to run iterative inference without additional privacy cost.

\para{Main contributions.}
The main contributions of this work are as follows:
\begin{itemize}
    \item We derive noise-aware DP Bayesian inference for GLMs based on approximate sufficient statistics from low-order polynomials.
    \vspace{-0.2cm}
    \item We prove tight $(\epsilon, \delta)$-DP bounds for releasing the approximate sufficient statistics.
    \vspace{-0.2cm}
    \item We demonstrate accurate privacy-preserving inference of which regression coefficients are significantly different from zero.
    \vspace{-0.2cm}
    \item We demonstrate high degree of similarity between the privacy-preserving and non-private posterior distributions for GLMs even under strong privacy for moderately sized data. 
\end{itemize}

\section{Background}

\subsection{Differential Privacy (DP)}
\label{sec:DP}
Assume a generic dataset $\bD \in \mathbb{R}^{N \times d}$ containing $d$-dimensional records of $N$ individuals. We define neighbourhood relation $\bD \sim \bD'$ when $\bD'$ can be obtained from $\bD$ by replacing a single record. \citet{DworkMNS06} proposed the following notion:

\begin{definition}
For $\epsilon \geq 0, \delta \geq 0$, a randomized mechanism $\mathcal{M}$ satisfies $(\epsilon,\delta)$-differential privacy if for all neighbouring datasets $\bD \sim \bD'$, and for all outputs $O \subseteq Range(\mathcal{M})$, the following constraint holds: 
\begin{equation}\label{eq:dp}
    \Pr[\mathcal{M}(\bD) \in O] \leq  \exp(\epsilon) \times \Pr[\mathcal{M}(\bD') \in O]  + \delta.
\end{equation}
\end{definition}

Lower values of $\epsilon$ and $\delta$ provide stronger privacy.
\par
The condition~\eqref{eq:dp} can often be satisfied, for example, by adding Gaussian noise to a function of the dataset
such that every individual's contribution is masked.
A key property of DP is its robustness to post-processing: the privacy loss of $\mathcal{M}$ cannot be increased by applying any randomized function independent of the data to $\mathcal{M}$'s output.

The concept of \emph{sensitivity} measures the worst-case impact of an individual's record on the output of a function.

\begin{definition}
The $L_2$-sensitivity $\Delta_{\bt}$ of a function $\bt:\mathbb{R}^{N \times d} \rightarrow \mathbb{R}^{m}$ is defined as $
\Delta_{\bt} = \max_{\bD \sim \bD'}|| \bt(\bD) - \bt(\bD')||_{2}$.
\end{definition}

\para{Analytic Gaussian Mechanism~\cite{analyticGaussian}.}
\citet{analyticGaussian} proposed an algorithmic noise calibration strategy based on the Gaussian cumulative density function (CDF) to obtain a mechanism that adds the least amount of Gaussian noise needed for $(\epsilon, \delta)$-DP.  

\begin{definition}{(Analytic Gaussian Mechanism)} For any $\epsilon \geq 0, \delta \in [0,1]$, a mechanism $\mathcal{M}(\bD) = \bt(\bD) + \zeta$ with sensitivity $\Delta_{\bt}$ satisfies $(\epsilon,\delta)$-DP with $\zeta \sim \mathcal{N}(0,\sigma^{2} I)$ iff
\begin{equation} \label{eq:agm}
\begin{aligned} 
    \Phi\left( \frac{\Delta_{\bt}}{2\sigma}- \frac{\epsilon \sigma}{\Delta_{\bt}} \right) -  \exp(\epsilon)\Phi\left(- \frac{\Delta_{\bt}}{2\sigma}- \frac{\epsilon \sigma}{\Delta_{\bt}} \right)  \leq \delta.
\end{aligned}
\end{equation}
\end{definition}

We use the
implementation based on Algorithm 1 of \citet{analyticGaussian} to find a minimal $\sigma$ that satisfies the condition~\eqref{eq:agm}.

\subsection{Bayesian inference based on sufficient statistics} 
For certain statistical models, the information about data needed for parameter inference can be captured by a limited number of \emph{sufficient statistics}. Sufficient statistics are available for exponential family models, such as linear regression.
For Bayesian linear regression, the sufficient statistics are $\bs=\sum_{i=1}^{N} \bt(\bx_i,y_i)=[\bX^{T}\bX,\bX^{T}\by,\by^{T}\by]$, where $\bX \in \mathbb{R}^{N \times d}$. With access to $\bs$ ($\mathcal{O}(N)$ operation), we can evaluate the total log-likelihood $\log(\prod_{i=1}^{N}\Pr[y_i | \bx_i, \btheta]) =\sum_{i=1}^{N} \log(y_i | \bx_i, \btheta)$ or its gradients in nearly a constant time. As a consequence, running time of a training algorithm taking $K$ passes over data reduces substantially to $\mathcal{O}(N+K)$ from $\mathcal{O}(NK)$.

\subsection{Privacy-preserving posterior inference with sufficient statistics} 
Consider a conditional probabilistic model $\Pr[\by | \btheta, \bX]$ where
the information in $(\bX,\by)$ needed for inference of $\btheta$ can be represented by sufficient statistics $\bs \in \mathbb{R}^{m}$. Let $\btheta \in \mathbb{R}^{d}$ denote the regression parameters. In order to adapt the model for DP,
we need to make additional assumptions on $\bX$. Following \citet{BS:19}, we assume $\bX \sim \mathcal{N}(\bzero, \bSigma)$, where $\bSigma \in \mathbb{R}^{d \times d}$ is an additional parameter denoting the input data covariance matrix. We guarantee privacy by operating solely on the perturbed sufficient statistics, denoted by $\bz \in \mathbb{R}^{m}$. The noise-aware joint posterior distribution of the model parameters $\btheta, \bSigma$, given noisy sufficient statistics $\bz$ is  

\begin{equation}
\begin{aligned}
 \Pr[\btheta,& \bSigma| \bz] \propto \Pr[\btheta, \bSigma, \bz] =\int_{\bs} \Pr[\btheta,\bSigma,\bs, \bz] \diff\bs \\ &= \int_{\bs} \Pr[\btheta] \Pr[\bSigma] \Pr[\bs\mid \btheta,\bSigma]  \Pr[\bz \mid \bs] \diff\bs,   \label{eq:priv_model_2} 
\end{aligned}
\end{equation}

where $\Pr[\btheta]$ and $\Pr[\bSigma]$ are the priors for model parameters and the privacy inducing noise is quantified by the term $\Pr[\bz \mid \bs]$. Figure~\ref{fig:glm} depicts above model. The remaining challenge is to define the probabilistic model for the latent sufficient statistics $\bs$.

\para{Normal approximation of $\bs$.} We obtain the sufficient statistic distribution $\Pr[\bs \mid \btheta,\bSigma]$ by marginalizing over the data: $\Pr[\bs\mid \btheta,  \bSigma] = \int_{\bX,\by: \bt(\bX,\by)=\bs} \Pr[\bX,\by \mid \btheta,  \bSigma] \diff\bX \diff\by$. However, this integral is in general intractable due to (possibly infinite) number of combinations of $\bX \in \mathbb{R}^{N \times d},\by \in \mathbb{R}^{N}$ that produce the sufficient statistic $\bs$. As $\bs$ is a sum of individual sufficient statistics $\bt(\bx_i, y_i)$, \citet{BS:19} proposed to approximate $\Pr[\bs \mid \btheta, \bSigma]$ as a multivariate normal distribution $\mathcal{N}(\bs \mid N\bmu_{\bs}, N\bSigma_{\bs})$, according to the central limit theorem,\footnote{The central limit theorem ensures the asymptotic accuracy of this approximation.} with mean $\bmu_{\bs}=\mathbb{E}[\bs]$ and covariance $\bSigma_{\bs}=\cov[\bs]$.

\subsection{Generalized linear models}
Generalized linear models~(GLMs, \citealp{glm}) include some of the most commonly used statistical models. GLMs extend linear regression by allowing for the possibility of more general outcome distributions such as binary, count,
and heavy-tailed observations, and using a linear model for the
mean parameter of the outcome distribution.

Denoting the input $\bx \in \mathbb{R}^{d}$ and unknown regression parameter $\btheta \in \mathbb{R}^{d}$, GLMs use a link function $g$ to associate the linear model $\bx^{T}\btheta$ to the mean of a response variable $y \in \mathbb{R}$ as $\mathbb{E}[y]=\mu=g^{-1}(\bx^{T}\btheta)$, where $g^{-1}:\mathbb{R} \rightarrow \mathbb{R}$ is the inverse link function. Typical examples of GLMs include logistic regression with binomial model for $y$ and logistic link $g(\mu) = \log(\frac{\mu}{1-\mu})$, as well as Poisson regression with Poisson distribution for $y$, usually combined with the log-link $g(\mu) = \log(\mu)$.

GLMs do not generally admit finite sufficient statistics.
\citet{PASS:17} propose the PASS-GLM framework to develop polynomial approximations of degree $m$
to GLMs that admit sufficient statistics. Sufficient statistics for such a polynomial
approximation can be seen as summary statistics or approximate sufficient statistics
for the original GLM.

\begin{figure}[ht]
\centering

\begin{tikzpicture}
\tikzstyle{main}=[circle, minimum size = 5, thick, draw =black!80, node distance = 4mm]
\tikzstyle{connect}=[-latex,  thick]
\tikzstyle{box}=[rectangle, draw=black!100,fill=blue!40]
  \node[main,line width=0.1mm] (Sigma) [] { $\bSigma$};  \node[main,line width=0.1mm] (theta) [left=of Sigma] { $\btheta$}; 
 
  \node[main,line width=0.1mm] (X) [below=of Sigma][] {$\bX$}; 
  \node[main,line width=0.1mm] (Y) [below=of theta][] {$\by$}; 
  \node[main,line width=0.1mm] (s) at ( -0.8,-2.0) [circle,draw=black!100] {$\bs$};
  \node[main] (z) [below=of s][fill=black!40] {$\bz$};     
  \path (Sigma) edge [connect] (X);
  \path (X) edge [connect] (Y);
  \path (theta) edge [connect] (Y);
  \path (X) edge [connect] (s);
  \path (Y) edge [connect] (s);
  \path (s) edge [connect] (z);   
 \node[rectangle, inner sep=3.0mm, fit= (z)(s)(theta)(Sigma),draw=black, xshift=0mm] {};
\end{tikzpicture}
\caption{Differentially Private Bayesian GLM} \label{fig:glm}
\end{figure}
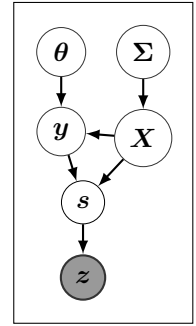

\section{Privacy-preserving Bayesian inference for GLMs}

\subsection{Model and problem formulation}
We consider the usual centralised DP setting where a dataset $\bD= \{\bX,\by\}=\{\bx_i,y_i\}_{i=1}^{N}$ is a multiset of $N$ observations. 
Motivated by the PASS-GLM approach, we summarise the data using low-order moments of $\bD$ that are useful for inference of the particular GLM.

The data holder computes the summary statistics $\bs$ for $\bD$ and releases a perturbed version $\bz = \bs + \zeta =  \sum_{i=1}^{N}\bt(\bx_i,y_i)+\zeta$ with noise $\zeta$ drawn from the analytic Gaussian mechanism defined in Section~\ref{sec:DP}.
The data holder also releases the details of the target GLM and DP mechanism. With access to $\bz$ and the noise mechanism, our goal is to design a noise-aware method for inferring the posterior distributions of the model parameters $\btheta$ and $\bSigma$ for representative GLMs --- logistic and Poisson regression.

\begin{figure*}[t]
    \centering        
$\bt(\boldsymbol{x},y)= \Big[1,   x_1y , x_2y,x_3y, x_4y,  x^{2}_1y^{2} ,  x^{2}_2y^{2},x^{2}_3y^{2}, x^{2}_4y^{2}, x_1 x_2 y^{2},x_1 x_3y^{2} ,x_1  x_4y^{2},x_2 x_3y^{2},x_2 x_4y^{2}, x_3  x_4 y^{2} \Big]$ 
    \caption{An example of a second order (i.e. $m=2$) approximate sufficient statistic $\bt (\bx,y)$ for logistic regression when $d=4$.}
    \label{fig:ss}
\end{figure*}

\subsection{Normal approximation of the summary statistics}\label{sec:NA}
Following \citet{BS:19}, we model $\bx_i \in \mathbb{R}^{d}$ as $\bx_i \sim \mathcal{N}^{d}(\bzero,\bSigma)$,
with an unknown covariance matrix $\bSigma \in \mathbb{R}^{d \times d}$.

Recall that we approximate the distribution of the summary statistics $\Pr[\bs \mid \btheta,\bSigma]$ with a normal approximation $\Pr[\bs \mid \btheta,\bSigma] \approx \mathcal{N}(\bs \mid N \bmu_{\bs}, N \bSigma_{\bs})$.
Next we show how to obtain the mean and covariance of the summary statistics analytically under our model. Note that we could also estimate these moments numerically, which might be easier for certain models, but creates an extra computational cost.

\para{Closed forms for the entries of $\bmu_{\bs}$ and $\bSigma_{\bs}$ for logistic regression.} 

Assuming $y_i \in \{-1, 1\}$, the logistic regression model can be written as
\begin{equation}
    \Pr[y_i | \bx_i, \btheta] = \sigma(\bx_i^T \btheta)^{\frac{1+y_i }{2}} (1-\sigma(\bx_i^T \btheta))^{\frac{1-y_i}{2}}, 
    \label{eq:logistic}
\end{equation}
where $\sigma(x) = \frac{1}{1+\exp(-x)}$ denotes the sigmoid function. Motivated by the success of a second order PASS-GLM approximation for logistic regression, we use the same summary statistic, presented in Fig.~\ref{fig:ss}.
We begin our calculations by computing the $a$-th moment of $y$:

\begin{equation} \label{eq:higher_moment_y}
\begin{aligned}
\mathbb{E}[y^{a} \mid \bx] &= (-1)^{a} \left(\tfrac{\exp(\bx^{T} \btheta)}{1+\exp(\bx^{T}\btheta)}\right) + (1)^{a}\tfrac{1}{1+\exp(\bx^{T}\btheta)} \\ 
&=
\begin{cases}
        1, &\text{ for even a's} \\
        \tfrac{1-\exp(\bx^{T} \btheta)}{1+\exp(\bx^{T} \btheta)}, & \text{ for odd a's} .
\end{cases} 
\end{aligned}
\end{equation}

The entries in $\bmu_{\bs}$ and $\bSigma_{\bs}$ are indexed by the non-negative integer exponents $a,b,c,d$ such that $a+b=m, c+d = m,m \leq 2$. Therefore, for all $i,j$th co-ordinates in $\bx$, the corresponding entries in $\bmu_{\bs}$ and $\bSigma_{\bs}$ can be populated using Equation~\ref{eq:higher_moment_y} as below:
\begin{enumerate}
    \item $\mathbb{E}[x_i^{a}y^{a}x_j^{b}y^{b}] = \mathbb{E}_{\bx}[x_i^{a}x_j^{b} \mathbb{E}_{y | \bx} [y^{a+b}]] \\
    =\begin{cases}
        \mathbb{E}_{\bx}[x_i^{a}x_j^{b}], & \text{ for even a+b's} \\
        \mathbb{E}_{\bx}[x_i^{a}x_j^{b}\frac{1-\exp(\bx^{T}\btheta)}{1+\exp(\bx^{T}\btheta)}] & \text{ for odd a+b's} 
    \end{cases}
    $
    
     \item $\cov[x^{a}_i x^{b}_j y^{a+b},x^{c}_k x^{d}_l y^{c+d}] = \mathbb{E}_{\bx}[x_i^{a} x_j^{b} x_k^{c} x_l^{d} y^{a+b+c+d}]  - \mathbb{E}[x_i^{a} x_j^{b} y^{a+b}] \mathbb{E}[x_k^{c} x_l^{d} y^{c+d}]$

     We can further simplify the covariance entries based on the parity of $a+b+c+d$.
    \begin{enumerate}
        \item When both $a+b$ and $c+d$ are even.
        \begin{equation*}
        \begin{aligned}
            & \cov[x^{a}_i x^{b}_j y^{a+b},x^{c}_k x^{d}_l y^{c+d}] \\ &= \mathbb{E}_{\bx}[x_i^{a} x_j^{b} x_k^{c} x_l^{d} ]-\mathbb{E}_{\bx}[x_i^{a} x_j^{b} ] \mathbb{E}_{\bx}[x_k^{c} x_l^{d}]
        \end{aligned}
        \end{equation*}
        \item When both $a+b$ and $c+d$ are odd.
        \begin{equation*}
        \begin{aligned}
            & \cov[x^{a}_i x^{b}_j y^{a+b},x^{c}_k x^{d}_l y^{c+d}] \\ &= \mathbb{E}_{\bx}[x_i^{a} x_j^{b} x_k^{c} x_l^{d} \Big( \frac{1-\exp(\bx^{T} \btheta)}{1+\exp(\bx^{T} \btheta)}\Big) ] \\ &- \mathbb{E}_{\bx}[x_i^{a} x_j^{b} \Big( \frac{1-\exp(\bx^{T} \btheta)}{1+\exp(\bx^{T} \btheta)}\Big) ] \\ & \mathbb{E}_{\bx}[x_k^{c} x_l^{d} \Big( \frac{1-\exp(\bx^{T} \btheta)}{1+\exp(\bx^{T} \btheta)}\Big)]
        \end{aligned}
        \end{equation*}
        \item When $a+b$ is even and $c+d$ is odd.
        \begin{equation*}
        \begin{aligned}
        &    \cov[x^{a}_i x^{b}_j y^{a+b},x^{c}_k x^{d}_l y^{c+d}] \\ &= \mathbb{E}_{\bx}[x_i^{a} x_j^{b} x_k^{c} x_l^{d} \Big( \frac{1-\exp(\bx^{T} \btheta)}{1+\exp(\bx^{T} \btheta)}\Big)] \\ & -\mathbb{E}_{\bx}[x_i^{a} x_j^{b} ] \mathbb{E}_{\bx}\Big[x_k^{c} x_l^{d} \Big( \frac{1-\exp(\bx^{T} \btheta)}{1+\exp(\bx^{T} \btheta)}\Big)\Big]
         \end{aligned}
        \end{equation*}
        \item Case $a+b$ is odd and $c+d$ is even follows identically from the previous case.
    \end{enumerate}
\end{enumerate}
\para{Taylor series expansion.} 
The non-linear term $\frac{1-\exp(\bx^{T} \btheta)}{1+\exp(\bx^{T} \btheta)}$ makes the expectation in previous formulas intractable. We approximate this using a truncated Taylor series.
The first two terms of the Taylor series approximation for $\frac{1-\exp(\bx^{T}\btheta) }{1+\exp(\bx^{T}\btheta)}$ are  $-\frac{\bx^{T}\btheta}{2}+\frac{(\bx^{T}\btheta)^{3}}{24}$.
This approximation is reasonably accurate as long as $\bx^{T}\btheta \in [-1,1]$. We now approximate one of the expectations from the cases above:
\begin{align*}
& \mathbb{E}_{\bx}[x_i^{a} x_j^{b} x_k^{c} x_l^{d} \Big( \frac{1-\exp(\bx^{T} \btheta)}{1+\exp(\bx^{T} \btheta)}\Big)] \\ & \approx \mathbb{E}_{\bx}[x_i^{a} x_j^{b} x_k^{c} x_l^{d} \Big[-\dfrac{\bx^{T}\btheta}{2}+\dfrac{(\bx^{T}\btheta)^{3}}{24}\Big]] \\ &= -\frac{\mathbb{E}_{\bx}[x_i^{a} x_j^{b} x_k^{c} x_l^{d} (\bx^{T}\btheta) ]}{2} + \frac{\mathbb{E}_{\bx}[x_i^{a} x_j^{b} x_k^{c} x_l^{d} (\bx^{T}\btheta)^{3} ]}{24} \\ &= \frac{-\sum_{n=1}^{d} \theta_{n} \mathbb{E}_{\bx}[x^{a}_{i} x^{b}_{j}x^{c}_{k}x^{d}_{l} x_{n} ]}{2} + \\ &  \frac{\sum_{\be:\sum_{o=1}^{d}e_{o}=3} {d  \choose \be} (\prod_{n=1}^{d} \theta_{n}^{e_n})  (\mathbb{E}_{\bx}[x^{a}_{i} x^{b}_{j}x^{c}_{k}x^{d}_{l} \prod_{n=1}^{d}  x_{n}^{e_n}])  }{24}.
\end{align*}
In the derivation above, the term $(\bx^{T}\btheta)^{3}$ is expanded using the multinomial theorem.  We can approximate $\mathbb{E}_{\bx}[x_{i}^a x_{j}^b \Big(\frac{1-\exp(\bx^{T} \btheta)}{1+\exp(\bx^{T} \btheta)}\Big) ]$ using similar calculations. The corresponding calculations for Poisson regression are found in the Supplement. 

\para{Evaluation of higher-order Gaussian moments.} Now that our closed forms for the entries of $\bmu_{\bs}$ and $\bSigma_{\bs}$ only include the sum of monomials, what remains is the actual evaluation of these Gaussian moments. We use the \emph{Isserlis' theorem}~\cite{isserlis} to compute these moments. This theorem presents even-degree moments of a zero-centered multivariate Gaussian variable as a sum of products of $\bSigma$ entries. 

\begin{theorem} {Isserlis' theorem~\cite{isserlis}.}\label{thm:isserlis} Let $\bx \sim \mathcal{N}^{d}(\bzero,\bSigma)$ be a $d-$dimensional random variable. Then
\begin{equation*}
\begin{aligned}
    \mathbb{E}[x_1,\cdots,x_d]= \sum_{p \in P^{2}_{d}} \prod_{\{i,j\} \in p} \mathbb{E}[x_ix_j]=\sum_{p \in P^{2}_{d}} \prod_{\{i,j\} \in p} \Sigma_{ij},
\end{aligned}
\end{equation*}
 where $d$ is assumed to be an even number and $P$ is the set of all possible ways of partitioning $\{1,\cdots,d\}$ in to pairs $\{i,j\}$. For odd d's, $\mathbb{E}[x_1,\cdots,x_d]=0$. 
\end{theorem}

A few examples of moment's calculations using Isserlis' theorem are found in the Supplement.  

\para{Computational complexity for moment population.} For a normal approximation of a second order summary statistics, we require moments of degree $2,4$, and $6$ to populate $\frac{(d+1)^{2}(d+2)^{2}}{8} + \frac{(d+1)(d+2)}{2} \in \mathcal{O}(\frac{d^{4}}{8})$ unique entries in $\bmu_{\bs},\bSigma_{\bs}$.
Under a normal model for $\bX$, these can all be computed from the covariance of $\bX$.

\begin{proposition}
\label{prop:isserlis_count}
To evaluate a moment with degree $2k$, Theorem~\ref{thm:isserlis} generates $\frac{(2k-1)!}{2^{k-1}(k-1)!}$ partitions (number of summands) each containing $k$ entries.
\end{proposition}

Applying Proposition~\ref{prop:isserlis_count}, we calculate that for a single degree $2,4$ and $6$ moment, we need to perform at-most 1, 6, and 45 unique multiplications. However, modern hardware can compute each moment in nearly a constant time with clever caching and indexing tricks. So a very loose upper bound on the order of operations performed in each iteration is $\mathcal{O}(d^{4})$.

\subsection{Satisfying DP}
The last step in our model is to define $\Pr[\bz|\bs]$. In order to bound the global sensitivity, we assume that each input instance has a bounded $L_2$-norm, i.e. $||\bx||_2 \leq R$.

\subsubsection{Logistic regression}
\label{sec:sensitivity}
\para{Sensitivity analysis.} Recall that the approximate sufficient statistics for logistic regression contain both linear and quadratic terms.
To this end, we define the functions  $t_1:\mathbb{R}^{ d} \rightarrow \mathbb{R}^{d}$ and
$t_2:\mathbb{R}^{ d} \rightarrow \mathbb{R}^{{d + 2 \choose 2}}$ as
\begin{equation} \label{eq:t_12}
    \begin{aligned}
        t_1(\bx) &= \bx, \\
        t_2(\bx) &= \begin{bmatrix} x_1^2  \! &  \! \! \ldots \! \! & \!  x_d^2 
         &  \! \sqrt{2}  \! x_1x_2 & \! \! \ldots \! \! &  \! \sqrt{2} x_{d-1} x_d \end{bmatrix}^T.
    \end{aligned}
\end{equation}
Using this notation, the approximate sufficient statistics are given as $[1, y t_1(\bx), y^2 t_2(\bx)]$, which 
due to $y \in \{-1, 1\}$ yields $t(\bx, y) = [1, yt_1(\bx), t_2(\bx)]$.
We consider a Gaussian mechanism where we release the linear and quadratic terms simultaneously.
When compared to individual releases of
$y t_1(\bx)$ and $t_2(\bx)$, this leads to a better utility.

\begin{lemma} \label{lem:t_12}
Let $t_1$ and $t_2$ be defined as in \eqref{eq:t_12} and let $\sigma_1,\sigma_2>0$.
Consider the mechanism 
$$
\mathcal{M}(\bx) = \begin{bmatrix} y t_1(\bx) \\ t_2(\bx)  \end{bmatrix} + 
\mathcal{N}\left(0, \begin{bmatrix} \sigma_1^2 I_d & 0 \\ 0 & \sigma_2^2 I_{d_2} \end{bmatrix} \right),
$$
where $d_2 = {d + 2 \choose 2}$. 
Assuming $|| \bx ||_2 \leq R$, the tight $(\epsilon,\delta)$-DP for $\mathcal{M}$ 
is obtained by considering a Gaussian mechanism with noise variance $\sigma_1^2$ and sensitivity
$$
\Delta = \sqrt{  \frac{\sigma_2^2}{2 \sigma_1^2} + 2 R^2 + 2 \frac{\sigma_1^2}{ \sigma_2^2} R^4    }.
$$
\end{lemma}

\begin{proof}
For the first order terms we have
 \begin{equation} \label{eq:fo}
 \begin{aligned}
     || yt_1(\bx) - y't_1(\bx') ||^{2}_2 &= 
 ||y\bx||^2 + ||y'\bx'||^2 - 2 \langle y\bx, y'\bx' \rangle \\
 &=||\bx||^2 + ||\bx'||^2 - 2 yy'\langle \bx, \bx' \rangle \\
  & \leq 2 R^2 - yy'2 \langle \bx, \bx' \rangle,
 \end{aligned}
 \end{equation} 
 and for the second order terms (see the Supplements)
  \begin{equation} \label{eq:so}
   \begin{aligned}
 || t_2(\bx) - t_2(\bx') ||^{2}_2 &= ||\bx||^4 + ||\bx'||^4 - 2 \langle \bx, \bx' \rangle^2 \\
   & \leq 2R^4 - 2 \langle \bx, \bx' \rangle^2.
  \end{aligned}
 \end{equation}
 We have that
 \begin{equation*}
     \begin{aligned}
     \mathcal{M}(\bx) & \sim \begin{bmatrix} yt_1(\bx) \\ t_2(\bx) \end{bmatrix} + \mathcal{N}\left(0, \begin{bmatrix} \sigma_1^2 I_d & 0 \\ 0 & \sigma_2^2 I_{d_2} \end{bmatrix} \right) \\
     &  \sim \begin{bmatrix} I_d & 0 \\
 0 &  \frac{\sigma_2}{\sigma_1} I_{d_2} 
 \end{bmatrix} \left( \begin{bmatrix} y t_1(\bx) \\ \frac{\sigma_1}{\sigma_2} t_2(\bx) \end{bmatrix} +   \mathcal{N}(0, \sigma_1^2 )  \right). \\
     \end{aligned}
 \end{equation*}
 Removing the constant scaling, we see that
 it is equivalent to consider the mechanism
 $
 \mathcal{M}(\bx) = \begin{bsmallmatrix} y t_1(\bx) \\ \frac{\sigma_1}{\sigma_2} t_2(\bx) \end{bsmallmatrix} +   \mathcal{N}(0, \sigma_1^2 ).
 $
 From \eqref{eq:fo} and \eqref{eq:so} we see that for the sensitivity of the function 
 $ \begin{bsmallmatrix} y t_1(\bx) \\ \frac{\sigma_1}{\sigma_2} t_2(\bx) \end{bsmallmatrix}$ we have
 $$
 \Delta \leq \sqrt{ -2 c t^2 - 2yy't +  2 c R^4 + 2R^2},
 $$
 where we denote $c=\tfrac{\sigma_1^2}{\sigma_2^2}$, $t = \langle \bx, \bx' \big\rangle$.
The bound has its maximum at $t = - \tfrac{yy'}{2c}$, which leads to the claim.
\end{proof}

\begin{corollary}
\label{cor:logistic_regression}
In the special case $R=1$ and $\sigma_1 = \sigma_2 = \sigma$, by Lemma~\ref{lem:t_12}, the optimal $(\epsilon,\delta)$ is obtained by considering the Gaussian mechanism
with noise variance $\sigma^2$ and sensitivity $\Delta = \sqrt{ 4 \tfrac{1}{2} }$.
\end{corollary}

\textbf{The general case.} When using a higher order polynomial (i.e. $m>2$), each $t_m(\bx)$ has to include all monomials of the form
$$
x_{i_1}^{k_1} \cdots x_{i_{m'}}^{k_{m'}}
$$
for all \emph{combinations} of positive integers $(k_1,\ldots,k_{m'})$ such that $k_1 + \ldots k_{m'} = m$.
Multiplying each monomial with the multinomial coefficient $\sqrt{\binom{m}{k_1,\ldots,k_{m'}}}$ and assuming $t_m(\bx)$ contains the monomials of order $m$, we find that (see the Supplements)
$$
|| t_m(\bx) - t_m(\bx') ||^{2}_2 = ||\bx||^{2m} + ||\bx'||^{2m} - 2 \langle \bx, \bx' \rangle^m.
$$
Finding the sensitivity upper bounds can then be carried out as in the case $m=2$.
For example, adding $\sigma$-noise to all terms, we need to bound the sensitivity of the function
$
\begin{bsmallmatrix} t_1(\bx)^T & \hdots &  t_m(\bx)^T \end{bsmallmatrix}^T
$
for which we have 
$$
\Delta^2 \leq \sum\nolimits_{i=1}^m  2 R^{2i} - 2 \langle \bx, \bx' \rangle^i.
$$
Maximum of the right hand side is found by minimising the polynomial $\sum_{i=1}^m t^i$. For example, for $m=6$, we find that the upper bound is attained for $\langle \bx, \bx' \rangle \approx -0.67$ and then 
$\Delta \approx \sqrt{12.72}$.

\subsubsection{Poisson regression}

In case of Poisson regression, in addition to $t_1(\bx)$ and $t_2(\bx)$,
sufficient statistics requires releasing $y \in \mathbb{N}^{\geq 0}$.
Similarly to Lemma~\ref{lem:t_12}, we can show the following:

\begin{lemma} \label{lem:t_12y}
Let $t_1$ and $t_2$ be defined as in \eqref{eq:t_12} and let $\sigma_1,\sigma_2,\sigma_3,\sigma_4>0$.
Suppose $||\bx||_2 \leq R_x$ and $y \leq R_y$.
Consider the mechanism 
$$
\mathcal{M}(\bx) = \begin{bsmallmatrix} t_1(\bx) \\ t_2(\bx) \\ y t_1(\bx) \\ y t_2(\bx)\end{bsmallmatrix} + \mathcal{N}\left(0, \begin{bsmallmatrix} \sigma_1^2 I_d & 0 & 0 & 0 \\ 0 & \sigma_2^2 I_{d_2} & 0 & 0 \\
0 & 0 & \sigma_3^2 I_{d_1} & 0 \\
0 & 0 & 0 & \sigma_4^2 I_{d_2} \end{bsmallmatrix} \right).
$$
Then, the tight $(\epsilon,\delta)$-DP for $\mathcal{M}$ 
is obtained by considering a Gaussian mechanism with noise variance $\sigma_1^2$
and sensitivity
$$
\Delta = \frac{ 2 (c_2 + c_4) R_x^2 + c_3 + 1 }{ \sqrt{2(c_2+c_4)} },
$$
where $c_2 = \tfrac{\sigma_1^2}{\sigma_2^2}$,
$c_3 = \tfrac{\sigma_1^2}{\sigma_3^2} R_y^2$ and
$c_4 = \tfrac{\sigma_1^2}{\sigma_4^2} R_y^2$.
\end{lemma}

\begin{corollary}
In the special case $R_x=1$ and $\sigma_1 = \ldots = \sigma_4 = \sigma$, the  optimal $(\epsilon,\delta)$ is obtained by considering the Gaussian mechanism with noise variance $\sigma^2$
and sensitivity $\Delta = \sqrt{ 4 \tfrac{1}{2}( 1 + R^{2}_y ) }$.
\end{corollary}

\section{Experiments}
In this Section we show experiments on logistic regression and leave experiments on Poisson regression to the Supplement.

\subsection{Default settings and implementation}
Throughout our experiments, we use the first two central moments of joint $(\bX, y)$ as the summary statistics. 
Both noise-aware and non-private baseline models for logistic regression are specified in Stan~\cite{Stan} using its Python interface. We use Stan's default \emph{No-U-Turn}~\cite{nuts} sampler, which is a variant of Hamiltonian Monte Carlo. 
We run 4 Markov chains in parallel and discard the first 50\% burn-in samples. We fix $R=1$ to use Corollary~\ref{cor:logistic_regression}. The Gelman-Rubin convergence statistic~\cite{R_hat:1998} was consistently below 1.1 for all Stan experiments. For brevity we provide comparisons only for the $\btheta$'s in our figures.

\begin{figure*}[t]
    \centering        
    \includegraphics[width=0.8\textwidth]{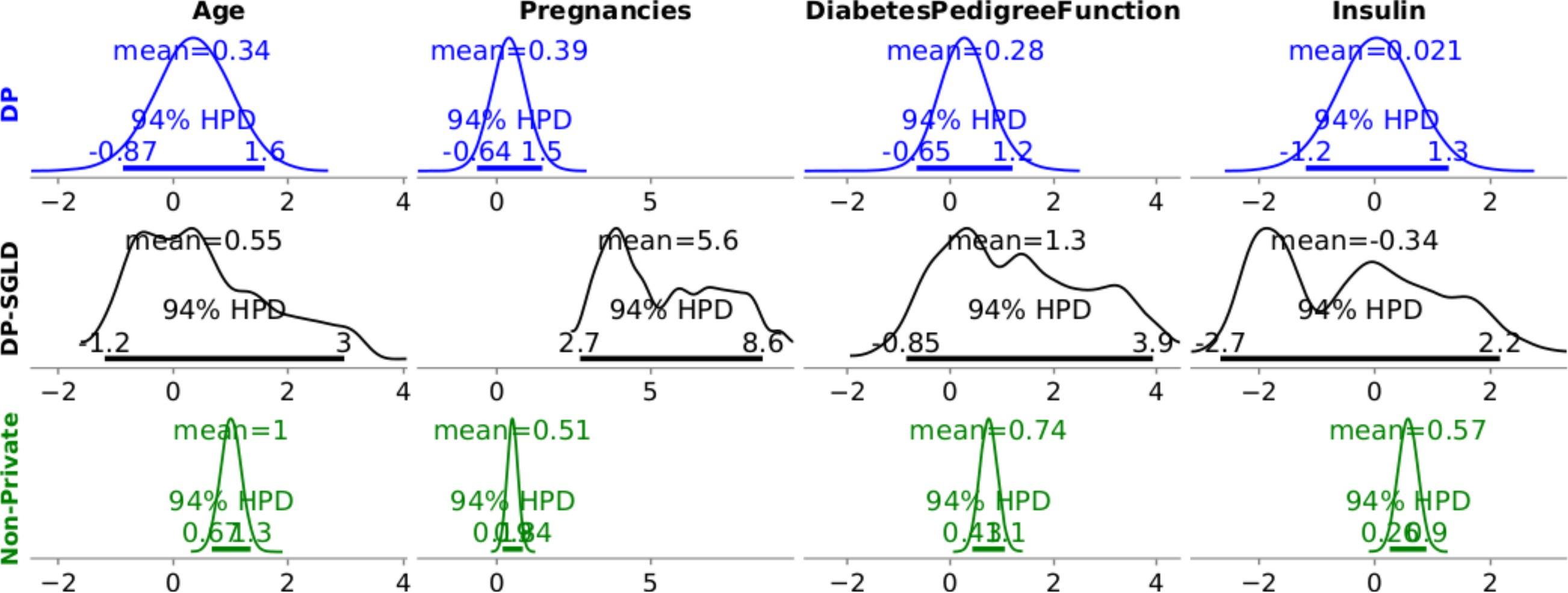}
    \includegraphics[width=0.8\textwidth]{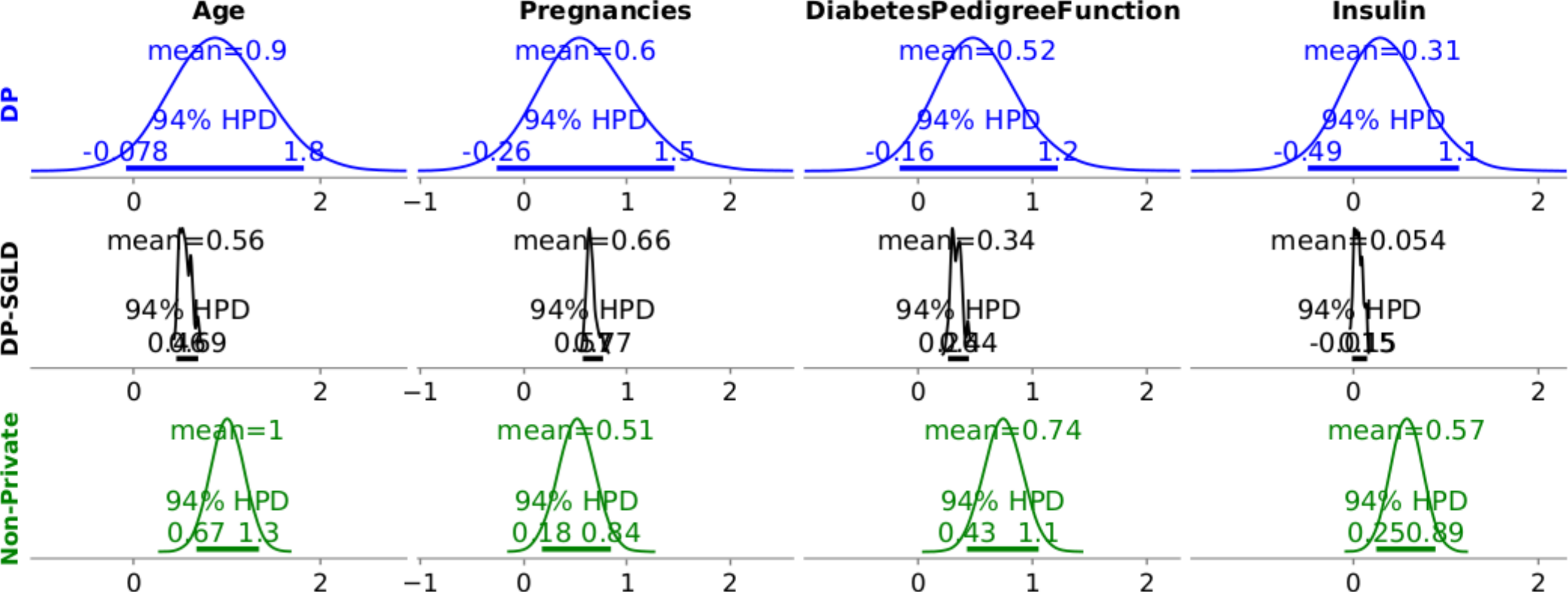}
    \caption{Comparison of differentially private posteriors from our model (blue) and DP-SGLD (black) with non-private posteriors for $\btheta$ for the diabetes~\cite{Kahn:94} dataset ($N=758$) for $\epsilon=0.1,\delta=10^{-5}$ (top) and $\epsilon=0.3,\delta=10^{-5}$ (bottom) after $10,000$ iterations. The batch size and the learning rate chosen for DP-SGLD were $28$ and $10^{-1}$. The posteriors from DP-SGLD are more biased and either exhibit a much higher variance or fail to quantify the expected uncertainty.}
    \label{fig:diabetes}
\end{figure*}

\begin{figure*}[t]
    \centering        
    \includegraphics[width=0.95\textwidth]{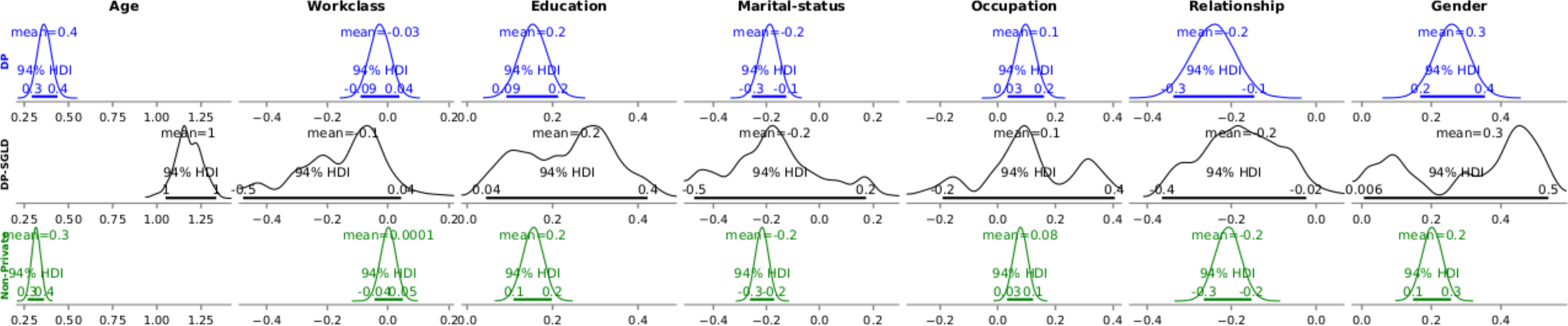}
    \caption{Comparison of differentially private posteriors from our model (DP) and DP-SGLD with non-private posteriors for $\btheta$ for randomly sampled $8000$ records in the adult dataset~\cite{Blake:98} for $\epsilon=0.1,\delta=10^{-5}$ after $10,000$ iterations. The batch size and the learning rate chosen for DP-SGLD were $89$ and $10^{-2}$. DP-SGLD  posteriors are biased and overestimate the uncertainty even at such large sample size.}
        \label{fig:adult}
\end{figure*}

\begin{figure*}[ht]
    \centering        \includegraphics[scale=0.45]{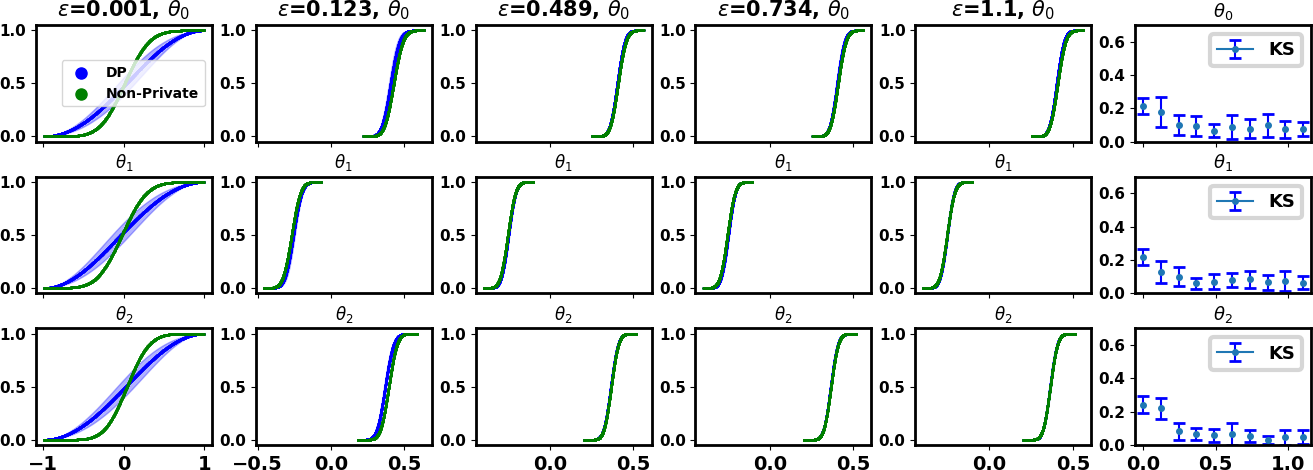}
    \caption{Comparison of differentially private and non-private empirical CDFs for $\btheta$'s posteriors for a synthetic dataset with $N=1000$ after $30,000$ iterations and $20$ repetitions while fixing $\delta=10^{-5}$ for various $\epsilon$ values. The right- most column shows the Kolmogorov-Smirnov scores between non-private and private empirical CDFs for various $\epsilon$ values.}
    \label{fig:cdf}
\end{figure*}

\para{Setting priors for model parameters $\btheta$ and $\bSigma$.}
For the data covariance matrix $\bSigma$ we gave a scaled LKJ~\cite{lkj:09} prior. We scale a positive definite correlation matrix from the LKJ correlation distribution of shape $\eta = 2$ from both sides with a diagonal matrix with $\mathcal{N}(0, 2.5)$ distributed diagonal entries. The probabilistic model is:

\begin{equation*}
\begin{aligned}
    \Omega &\sim \text{LKJ}(2), \quad
    \tau \sim \mathcal{N}(\mathbf{0}, 2.5 \cdot \mathbf{I}), \\
    \bSigma &= \text{diag}(\boldsymbol\tau) \, \Omega \, \text{diag}(\boldsymbol\tau).
\end{aligned}
\end{equation*}

In order to prevent the inference from sampling $\btheta$'s with large magnitudes, we gave the regression coefficients' orientation a uniform prior, and the squared norm a truncated Chi-square prior. We treat the upper-bound for the truncation as a hyper-parameter, which was set to 2 or 3 times the square of non-private $\btheta's$ norm. The exact probabilistic model is:

\begin{equation*}
\begin{aligned}
\mathbf{p} &\sim \mathcal{N}(\mathbf{0}, \mathbf{I}), \quad
    \rho \sim \chi^2(d) \\
    \btheta &= \sqrt{\max(\rho, s)} \frac{\mathbf{p}}{|| \mathbf{p} ||_2 }.
\end{aligned}
\end{equation*}

In reality, the upper bound $s$ could be obtained in a DP way or its approximate value could be known from domain expertise. The question of designing a better prior that does not require such truncation bound is left as a future exercise. 

\para{Private Baseline.} We compare our method with a Python implementation of general-purpose DP posterior inference  DP-SGLD~\cite{WYFS:2015,DP_SGLD:19}, using the Fourier accountant~\cite{Fourier_accountant} for tight DP accounting. DP-SGLD perturbs the scaled gradients in each step and therefore the noise scales proportional to the number of iterations, whereas we only pay a small upfront privacy cost for perturbing the sum of sufficient statistics and enjoy posterior samples for free. 
We run DP-SGLD with batch-size $\sqrt{N}$ (as suggested in \cite{ACGMMTZ:16}) for 10,000 iterations and discard the first 6000 samples as burn-in.

\subsection{Posterior recovery} 
\par We trained our logistic regression model on $8000$ pre-processed random records from the adult dataset~\cite{Blake:98} with features age, workclass, education, martial-status, occupation, relationship, and gender to predict whether a person's income exceeds $50K$ dollars. Furthermore, we also train on abridged version of a much smaller pre-processed diabetes~\cite{Kahn:94} dataset with dimensions age, number of pregnancies, diabetes pedigree function, and insulin levels to predict whether a person has diabetes or not. Figures~\ref{fig:diabetes} and \ref{fig:adult} illustrate the outcome of these experiments for DP parameters in a strong-privacy regime. We verify that for both datasets, private posteriors from our model are close to the non-private posteriors. The posteriors obtained from DP-SGLD are highly variable, ranging from highly biased to hugely underconfident to massively overconfident.

\subsection{Varying privacy requirements}
We now test the accuracy of our methods against a verity of privacy settings. 
Figure~\ref{fig:cdf} compares the posterior empirical cumulative distribution functions (CDFs) of private and non-private $\btheta$'s for a synthetic dataset with a randomly sampled positive definite non-identity co-variance matrix with true $\btheta$ as $[-0.9,-0.5,0.3]$. We see that the private and non-private CDFs are almost overlapping for $\epsilon > 0.1$. In the right-most column, we additionally plot the Kolmogorov-Smirnov scores (maximum absolute difference between two CDFs) for 10 equally spaced $\epsilon$ values in the range [0.001,1.1]. Once again, we note a general non-increasing trend.  

\par These results demonstrate that our model is accurate for datasets with small true $\btheta$'s for moderate to large sample sizes even when privacy requirements are strict. 

\section{Limitations}
In additional internal experiments on synthetic datasets, we studied accuracy as a function of $L_2$-norm of true $\btheta$, fixing other parameters, and found that the accuracy decreases sharply when $|| \btheta||_2 \geq 3$, when using the proposed approximations. We suspect this is due the truncation in the Taylor series, used to reduce the computational complexity. The bound on theta norm essentially means that the predicted probabilities can differ by at most a modest number of logistic units. While this may sound limiting, it may not be a serious issue for real datasets, where the signal is not very strong. With a proper prior, the model would most likely simply underfit in such cases. The implications on inference of the signs of the regression coefficients would also likely be limited.

\section{Concluding Remarks}
This work formulates a noise-aware model for GLMs for performing DP Bayesian inference and demonstrates its efficacy for datasets with regression coefficients of small magnitudes. Our method combines a normal approximation based on the central limit theorem with moment matching for perturbed low order data moments. We carry out a sensitivity analysis for the DP mechanisms which gives tight bounds and leads to high utility. This is also reflected in the experimental results on the logistic regression. Our sensitivity analysis also shows that we can increase utility by simultaneously releasing the linear and quadratic terms. Since computation of approximate sufficient statistics is a transformation of data, it seems possible to develop similar noise-aware models in distributed learning scenarios such as federated learning~\cite{FL} and local differential privacy~\cite{LDP:18}. 

\section{Acknowledgement}
This work was sponsored by the Academy of Finland; Grants 325573, 325572, 313124, 335516, and Flagship programme: Finnish Center for Artificial Intelligence, FCAI. We are grateful to the Aalto Science-IT project for their computational resources. We also wish to thank Daniel Sheldon for useful discussions and for the suggestion to use the PASS-GLM framework. 

\bibliography{papers}

\onecolumn

\section{Examples of moment calculation using Isserlis' Theorem}
\begin{theorem} {Isserlis' theorem~\cite{isserlis}.}\label{thm:isserlis} Let $\bx \sim \mathcal{N}^{d}(\bzero,\bSigma)$ be a $d-$dimensional random variable.
\begin{align}
    \mathbb{E}[x_1,\cdots,x_d]= \sum_{p \in P^{2}_{d}} \prod_{\{i,j\} \in p} \mathbb{E}[x_ix_j]=\sum_{p \in P^{2}_{d}} \prod_{\{i,j\} \in p} \Sigma_{ij}
\end{align}
 Where $d$ is assumed to be an even number and $P$ is the set of all possible ways of partitioning $\{1,\cdots,d\}$ in to pairs $\{i,j\}$. For odd d's, $\mathbb{E}[x_1,\cdots,x_d]=0$. 
\end{theorem}
For example, two $4$th order moments are computed using the Isserlis' theorem below.  
\begin{align*}
    \mathbb{E}[x^{2}_{1}x^{2}_{2}] &= \mathbb{E}[x_{1}x_{1}x_{2}x_{2}] \\ &= \Sigma_{11} \Sigma_{22}+\Sigma_{12} \Sigma_{12} +\Sigma_{12} \Sigma_{12} \\ &=\Sigma_{11} \Sigma_{22}+2\Sigma_{12} \Sigma_{12} \\ 
    \mathbb{E}[x^{2}_{1}x_{2} x_{3}] &= \mathbb{E}[x_{1}x_{1}x_{2} x_{3}] \\ &= \Sigma_{11} \Sigma_{23}+\Sigma_{12} \Sigma_{13} +\Sigma_{13} \Sigma_{12} \\ &=\Sigma_{11} \Sigma_{23}+2\Sigma_{12} \Sigma_{13}
\end{align*}

\section{Poisson regression with softplus as the link function}

s
We use the softplus link function $\mu = \mathbb{E}[y | \bx, \btheta] = \log(1+ \exp(\bx^{T}\btheta))$. The probability mass function is given by 
\begin{align}
    \Pr[y | \bx , \btheta] = \frac{\mu^{y} \exp(-\mu)}{y!} = \frac{(\log(1+ \exp(\bx^{T}\btheta)))^{y} \exp(-\log(1+ \exp(\bx^{T}\btheta)))}{y!}  
\end{align}
And the log-likelihood is given by 
\begin{align}\label{eq:ll_poisson}
\log \Big[ \Pr[y | \bx , \btheta]\Big] \propto y \log(\log(1+ \exp(\bx^{T}\btheta))) - \log(1+ \exp(\bx^{T}\btheta))
\end{align}
Now we find the normal approximation to each term in equation~\ref{eq:ll_poisson}.

\subsection{$\log(1+ \exp(\bx^{T}\btheta))$}

The second part of the likelihood i.e. $\log(1+\exp( \bx^{T} \btheta))$ is non-linear and its approximate computation requires a $m$ order polynomial expansion. The summary statistics are $\bt(\bx) = ([\bx]^{\boldsymbol{k}})_{\boldsymbol{k}} \forall \boldsymbol{k} \in \mathbb{N}^{d}: \sum_{j} \boldsymbol{k}_j = m' , m' \leq m$, where $[\bx]^{\boldsymbol{k}}= \prod_{j=1}^{d} \bx_j^{\boldsymbol{k}_{j}}$. An example of a second order (i.e. $m=2$) approximate summary statistic $\bt (\bx)$ for logistic regression when $d=4$ is given below.

\begin{align}
\bt(\boldsymbol{x})= \Big[1,   x_1 , x_2,x_3, x_4,  x^{2}_1 ,  x^{2}_2,x^{2}_3, x^{2}_4, x_1 x_2 ,x_1 x_3 ,x_1  x_4,x_2 x_3,x_2 x_4, x_3  x_4  \Big]    
\end{align}

The entries for $\bmu_{\bs}$ and $\bSigma_{\bs}$ for the second part of the summand is given by the following.
\begin{enumerate}
    \item $\mathbb{E}[x_i^{a}x_j^{b}]=  \left\{\begin{array}{lr}
        \mathbb{E}_{\bx}[x_i^{a}x_j^{b}] & \text{for even a+b's} \\
        0 & \text{ for odd a+b's} 
        \end{array}\right\}, a,b : a+b =m' \leq m$
    \item $\cov[x^{a}_i x^{b}_j ,x^{c}_k x^{d}_l ]= \mathbb{E}_{\bx}[x_i^{a} x_j^{b} x_k^{c} x_l^{d} ]-\mathbb{E}[x_i^{a} x_j^{b}] \mathbb{E}[x_k^{c} x_l^{d} ]  a,b: a+b=m', c+d = m',m' \leq m$. 
    \end{enumerate}
    According to the Isserlis' theorem, only the even degree moments are non zero. Therefore we have the following cases.
\begin{align*}    
     \cov[x^{a}_i x^{b}_j ,x^{c}_k x^{d}_l]=  \left\{\begin{array}{lr}
        \mathbb{E}_{\bx}[x_i^{a} x_j^{b} x_k^{c} x_l^{d} ]-\mathbb{E}_{\bx}[x_i^{a} x_j^{b} ] \mathbb{E}_{\bx}[x_k^{c} x_l^{d}] & \text{  a+b and c+d are even} \\
        \mathbb{E}_{\bx}[x_i^{a} x_j^{b} x_k^{c} x_l^{d} ] & \text{a+b and c+d are odd} \\       
        0 & \text{a+b+c+d is odd} 
        \end{array}\right\}
\end{align*}

\subsection{$y\log(\log(1+ \exp(\bx^{T}\btheta)))$} 
The entries of $\bt(\bx)$ for $\log(\log(1+ \exp(\bx^{T}\btheta)))$ are the same as those are for $\log(1+ \exp(\bx^{T}\btheta))$. So let's compute the entries for $\bmu_{\bs}$ for $y\bt(\bx)$. 
\begin{align*}
\mathbb{E}_{\bx}[\bt(\bx)y]=\mathbb{E}_{\bx}[ x_i^{a}x_j^{b} \mathbb{E}_{y | \bx}[ y]]=\mathbb{E}_{\bx} [ x_i^{a}x_j^{b} \log(1+\exp(\bx^{T} \btheta))] \approx \mathbb{E}_{\bx}\Big[ x_i^{a} x_{j}^{b}  \Big(\log(2)+\frac{\bx^{T} \btheta}{2}+ \frac{(\bx^{T}\btheta)^{2}}{8} - \frac{(\bx^{T} \btheta)^{4}}{192}\Big) \Big].
\end{align*}
Here $\log(1+\exp(\bx^{T} \btheta))$ is approximated using the first four terms of the its Taylor expansion.
Each monomial in the expansion of $(\bx^{T} \btheta)^{p}, p \in \mathbb{N}^{\geq 0}$ has degree $p$ and we expand these monomials using the multinomial theorem. Once again after applying the Isserlis' theorem, we have the following cases. 

    \begin{align}    
    \label{eq:mean_poisson}
    \mathbb{E}[x_i^{a}x_j^{b} \log(1+ \exp(\bx^{T}\btheta))]=  \left\{\begin{array}{lr}
        \mathbb{E}_{\bx}[x_i^{a}x_j^{b} \Big( \log(2) + \frac{(\bx^{T} \btheta)^{2}}{8} - \frac{(\bx^{T} \btheta)^{4}}{192} \Big)] & \text{ for even a+b's} \\
        \mathbb{E}_{\bx}[x_i^{a}x_j^{b} \Big( \frac{\bx^{T}\btheta}{2} \Big)] & \text{ for odd a+b's} 
        \end{array}\right\}, a,b : a+b =m' \leq m.
    \end{align}
Similarly the terms of $\bSigma_{\bs}$ for all $a,b: a+b=m'\leq m, c+d=m', m' \leq m$  are: 
\begin{align*}    
    \cov[x_i^{a}x_j^{b} y,x_k^{c}x_l^{d} y]  &=     \mathbb{E}_{\bx}[x_i^{a}x_j^{b}  x_k^{c}x_l^{d} \mathbb{E}_{y|\bx}[y^{2}]] - \mathbb{E}_{\bx}[x_i^{a}x_j^{b} \mathbb{E}_{y|\bx}[y]] \mathbb{E}_{\bx}[x_k^{c}x_l^{d} \mathbb{E}_{y|\bx}[y]] \\ &= \mathbb{E}_{\bx}[x_i^{a}x_j^{b}x_k^{c}x_l^{d} \log^{2}(1+\exp(\bx^{T}\btheta  ))] - \mathbb{E}_{\bx}[x_i^{a}x_j^{b} \log(1+\exp(\bx^{T}\btheta))]  \mathbb{E}_{\bx}[x_k^{c}x_l^{d} \log(1+\exp(\bx^{T}\btheta))].
    \end{align*}
The second part of this subtraction can be evaluated using Equation~\ref{eq:mean_poisson}. Next, we  evaluate $\mathbb{E}_{\bx}[x_i^{a}x_j^{b}x_k^{c}x_l^{d} \log^{2}(1+\exp(\bx^{T} \btheta  ))]$. Using the Taylor series expansion, 
\begin{align*}    
\log^{2}(1+\exp(\bx^{T} \btheta  )) \approx \log^{2}(2) + (\bx^{T} \btheta) \log(2) + \frac{(\bx^{T} \btheta)^{2}(1+\log(2))}{4}+\frac{(\bx^{T}\btheta)^{3}}{8}.
\end{align*}

The surviving even degree moments that we evaluate are.
    \begin{align*}   
    \mathbb{E}[x_i^{a}x_j^{b}x_k^{c}x_l^{d} \log^{2}(1+ \exp(\bx^{T}\btheta))]=  \left\{\begin{array}{lr}
        \mathbb{E}_{\bx}\Big[x_i^{a}x_j^{b} x_k^{c}x_l^{d} \Big( \log^{2}(2) + \frac{(\bx^{T}\btheta)^{2}(1+\log(2))}{4} \Big)\Big] & \text{ for even a+b+c+d's} \\
        \mathbb{E}_{\bx}\Big[x_i^{a}x_j^{b} x_k^{c}x_l^{d} \Big( (\bx^{T}\btheta) \log(2) +\frac{(\bx^{T}\btheta)^{3}}{8} \Big)\Big] & \text{ for odd a+b+c+d's} 
        \end{array}\right\}.
        \end{align*}    

These expressions are further simplified using the multinomial theorem once again. 
\section{Sensitivity results }

\subsection{Individual sensitivities}

\begin{lemma} 
Consider two vectors $\bx, \bx' \in \mathbb{R}^{d}$ such that $||\bx||_2 \leq R$ and $||\bx'||_2 \leq R$. Then, an elementary analysis shows that
\begin{equation*}
    \begin{aligned}
        || t_1(\bx)  -  t_1(\bx') ||_2 &\leq 2 R, \\
        || t_2(\bx)  -  t_2(\bx') ||_2 & \leq  \sqrt{2}R^{2}.
    \end{aligned}
\end{equation*}
\end{lemma}
When considering the Gaussian mechanism for releasing $t_1(\bx)$ or $t_2(\bx)$ such that element-wise Gaussian noises of variances $4R^2 \sigma^2$ and $2R^4 \sigma^2 $ are added to $t_1(\bx)$ and $t_2(\bx)$, respectively, their 
$(\varepsilon, \delta)$-analyses are equivalent to the analysis of Gaussian mechanism with sensitivity 1 and variance $\sigma^2$.
However, when releasing the linear and quadratic terms simultaneously, a better utility can be obtained.
To this end, we need the following relations.

\subsection{Second order terms}

In the case of $m=2$, we release the terms $(x_i^2, x_i x_j)$ with multipliers $(1, \sqrt{2})$. By rearranging, we have
\begin{equation*}
	\begin{aligned}
		\norm{t_2(x) - t_2(x')}_2^2 &= \sum\limits_{i=1}^d (x_i^2 - x_i'^2)^2 + 2 \sum\limits_{i>j} (x_i x_j -  x_i' x_j' )^2 \\
		&=  \sum\limits_{i} x_i^4 +  \sum\limits_{i} x_i'^4   - 2  \sum\limits_{i} x_i^2 x_i'^2 
		  + 2 \sum\limits_{i>j}  x_i^2 x_j^2  +   2 \sum\limits_{i>j}  x_i'^2 x_j'^2 - 4 \sum\limits_{i>j}  x_i x_j x_i' x_j' \\
		&=  \left( \sum\limits_{i} x_i^4 + 2 \sum\limits_{i>j}  x_i^2 x_j^2 \right) +  \left(  \sum\limits_{i} x_i'^4  +   2 \sum\limits_{i>j}  x_i'^2 x_j'^2 \right)
		 - 2  \sum\limits_{i} x_i^2 x_i'^2  - 4 \sum\limits_{i>j}  x_i x_j x_i' x_j'  \\
		&=  \left( \sum\limits_{i} x_i^4 + \sum\limits_{i \neq j}  x_i^2 x_j^2 \right) +  \left(  \sum\limits_{i} x_i'^4  +   \sum\limits_{i \neq j}  x_i'^2 x_j'^2 \right)
		 - 2  \sum\limits_{i} x_i^2 x_i'^2  - 2 \sum\limits_{i \neq j}  x_i x_j x_i' x_j'  \\
		&=  \norm{x}_2^4 +  \norm{x'}_2^4 - 2 \langle x, x' \rangle^2.
	\end{aligned}
\end{equation*}

\subsection{Third and fourth order terms}

We first illustrate the general case with the cases $m=3$ and $m=4$.
In the next section we describe the release mechanism and give its tight sensitivity for general $m$.

In the case of $m=3$, we release all the distinct terms of the form
$$
(x_i^3, x_i^2 x_j, x_i x_j x_k)
$$ 
with the corresponding multipliers $(1, \sqrt{3}, \sqrt{3})$. By rearranging, we have

\begin{equation*}
	\begin{aligned}
		\norm{t_3(x) - t_3(x')}_2^2 &= \sum\limits_{i=1}^d (x_i^3 - x_i'^3)^2 + 3 \sum\limits_{i \neq j} (x_i^2 x_j -  x_i'^2 x_j' )^2  + 
		3 \sum\limits_{i>j>k} (x_i x_j x_k -  x_i' x_j' x_k' )^2   \\
		&=  \sum\limits_{i} x_i^6 +  \sum\limits_{i} x_i'^6   - 2  \sum\limits_{i} x_i^3 x_i'^3
		+ 3 \sum\limits_{i \neq j}  x_i^4 x_j^2 + 3 \sum\limits_{i \neq j}  x_i'^4 x_j'^2 \\
		&	- 6 \sum\limits_{i \neq j}  x_i^2 x_j x_i'^2 x_j'   
			+ 3 \sum\limits_{i>j>k} x_i^2 x_j^2 x_k^2 + 3 \sum\limits_{i>j>k} x_i'^2 x_j'^2 x_k'^2 - 6  \sum\limits_{i>j>k} x_i x_j x_k x_i' x_j' x_k'  \\
		&= \left( \sum\limits_{i} x_i^6  + 3 \sum\limits_{i \neq j}  x_i^4 x_j^2 + 3 \sum\limits_{i>j>k} x_i^2 x_j^2 x_k^2  \right) + 
		    \left( \sum\limits_{i} x_i'^6  + 3 \sum\limits_{i \neq j}  x_i'^4 x_j'^2 + 3 \sum\limits_{i>j>k} x_i'^2 x_j'^2 x_k'^2  \right)   \\
		& - 2  \sum\limits_{i} x_i^3 x_i'^3 - 6 \sum\limits_{i \neq j}  x_i^2 x_j x_i'^2 x_j' - 6  \sum\limits_{i>j>k} x_i x_j x_k x_i' x_j' x_k' \\
		& = \norm{x}_2^6 +  \norm{x'}_2^6 - 2 \langle x, x' \rangle^3.
				%
	\end{aligned}
\end{equation*}

In the case of $m=4$, we release  all the distinct the terms of the forms
$$
(x_i^4,x_i^3 x_j, x_i^2 x_j^2, x_i^2 x_j x_k, x_i x_j x_k x_\ell)
$$
with the corresponding multipliers $(\sqrt{4},\sqrt{6},\sqrt{6},\sqrt{4})$. By rearranging, we have

\begin{equation*}
	\begin{aligned}
		\norm{t_4(x) - t_4(x')}_2^2 &= \sum\limits_{i=1}^d (x_i^4 - x_i'^4)^2 
		+ 4 \sum\limits_{i \neq j} (x_i^3 x_j -  x_i'^3 x_j' )^2
		+ 6 \sum\limits_{i > j} (x_i^2 x_j^2 -  x_i'^2 x_j'^2 )^2 \\ 
		& + 6 \sum\limits_{i \neq j, i \neq k, j \neq k} (x_i^2 x_j x_k -  x_i'^2 x_j' x_k' )^2  
		+ 4 \sum\limits_{i>j>k>\ell} (x_i x_j x_k x_\ell -  x_i' x_j' x_k' x_\ell' )^2  \\
		 &=  \bigg( \sum\limits_{i} x_i^8 +  \sum\limits_{i} x_i'^8   - 2  \sum\limits_{i} x_i^4 x_i'^4 \bigg)
 		+  \bigg( 4 \sum\limits_{i \neq j}  x_i^6 x_j^2 + 4 \sum\limits_{i \neq j}  x_i'^6 x_j'^2 \\
		& - 8 \sum\limits_{i \neq j} x_i^3 x_j x_i'^3 x_j' \bigg)
		+  \bigg( 6 \sum\limits_{i > j} x_i^4 x_j^4 + 6 \sum\limits_{i > j} x_i'^4 x_j'^4 - 12 \sum\limits_{i > j} x_i^2 x_j^2   x_i'^2 x_j'^2 \bigg) \\
		&  +  \bigg( 6 \sum\limits_{i \neq j, i \neq k, j \neq k} x_i^4 x_j^2 x_k^2 
		 + 6 \sum\limits_{i \neq j, i \neq k, j \neq k} x_i'^4 x_j'^2 x_k'^2 
		 - 12 \sum\limits_{i \neq j, i \neq k, j \neq k} x_i^2 x_j x_k  x_i'^2 x_j' x_k' \bigg) \\
	&	 +  \bigg( 4 \sum\limits_{i>j>k>\ell} x_i^2 x_j^2 x_k^2 x_\ell^2 
		 + 4 \sum\limits_{i>j>k>\ell} x_i'^2 x_j'^2 x_k'^2 x_\ell'^2
		 - 8  \sum\limits_{i>j>k>\ell} x_i x_j x_k x_\ell  x_i' x_j' x_k' x_\ell' \bigg) \\ 
		 &=  \bigg( \sum\limits_{i} x_i^8 + 4 \sum\limits_{i \neq j}  x_i^6 x_j^2  + 6 \sum\limits_{i > j} x_i^4 x_j^4 
		 + 4 \sum\limits_{i>j>k>\ell} x_i^2 x_j^2 x_k^2 x_\ell^2  \bigg) + \\
		 & + \bigg( \sum\limits_{i} x_i'^8 + 4 \sum\limits_{i \neq j}  x_i'^6 x_j'^2  + 6 \sum\limits_{i > j} x_i'^4 x_j'^4 
		 + 4 \sum\limits_{i>j>k>\ell} x_i'^2 x_j'^2 x_k'^2 x_\ell'^2   \bigg) \\
		 & - 2 \bigg(  \sum\limits_{i} x_i^4 x_i'^4 + 4 \sum\limits_{i \neq j} x_i^3 x_j x_i'^3 x_j'
		 +6 \sum\limits_{i \neq j, i \neq k, j \neq k} x_i^2 x_j x_k  x_i'^2 x_j' x_k' \\
 		 & + 6 \sum\limits_{i > j} x_i^2 x_j^2   x_i'^2 x_j'^2
		 +4 \sum\limits_{i>j>k>\ell} x_i x_j x_k x_\ell  x_i' x_j' x_k' x_\ell'		\bigg) \\
		 &=  \norm{x}_2^8 +  \norm{x'}_2^8 - 2 \langle x, x' \rangle^4.
	\end{aligned}
\end{equation*}

\subsection{General case}

For a general $m$, if we release each distinct $m$th order termof the form
$$
x_{i_1}^{k_1} \cdots x_{i_{m'}}^{k_{m'}},
$$
where $\sum_i k_i = m$, $1 \leq m' \leq m$, multiplied with the multinomial factor $\sqrt{\binom{m}{k_1,\ldots,k_{m'}}}$,
then the function $t_m(x)$ has the sensitivity
$$
\norm{t_m(x) - t_m(x')}_2^2 =  \norm{x}_2^{2m} +  \norm{x'}_2^{2m} - 2 \langle x, x' \rangle^m.
$$
This is shown similarly as above for $t_2(x), t_3(x)$ and $t_4(x)$. Namely, we have
\begin{equation*}
	\begin{aligned}
		\norm{t_m(x) - t_m(x')}_2^2 &=    \sum\limits_{k_1 + \ldots k_{m'} = m, \, 1 \leq m' \leq m} 
		\, \, \sum\limits_{i_1, \ldots i_{m'}}  \binom{m}{k_1,\ldots,k_{m'}}  ( x_{i_1}^{k_1} \cdots  x_{i_{m'}}^{k_{m'}}
		 - x_{i_1}'^{k_1} \cdots  x_{i_{m'}}'^{k_{m'}} )^2 \\
		 &=    \bigg( \sum\limits_{k_1 + \ldots k_{m'} = m, \, 1 \leq m' \leq m} 
		 	\, \,	  \sum\limits_{i_1, \ldots i_{m'}} 	\binom{m}{k_1,\ldots,k_{m'}} x_{i_1}^{2 k_1} \cdots  x_{i_{m'}}^{2 k_{m'}} \bigg) \\
		& + \bigg( \sum\limits_{k_1 + \ldots k_{m'} = m, \, 1 \leq m' \leq m} 
		 	\, \,	 \sum\limits_{i_1, \ldots i_{m'}} 	\binom{m}{k_1,\ldots,k_{m'}} x_{i_1}'^{2 k_1} \cdots  x_{i_{m'}}'^{2 k_{m'}}	\bigg)	  \\
		& - 2 \cdot \bigg( \sum\limits_{k_1 + \ldots k_{m'} = m, \, 1 \leq m' \leq m} 
		 	\, \,		 \sum\limits_{i_1, \ldots i_{m'}} \binom{m}{k_1,\ldots,k_{m'}} x_{i_1}^{k_1} x_{i_1}'^{k_1} \cdots 
				  x_{i_{m'}}^{k_{m'}}  x_{i_{m'}}'^{k_{m'}}  \bigg) \\	
		 &=  \norm{x}_2^{2m} +  \norm{x'}_2^{2m} - 2 \langle x, x' \rangle^m,
	\end{aligned}
\end{equation*}
where 
$
\sum\nolimits_{k_1 + \ldots k_{m'} = m, \, 1 \leq m' \leq m} 
$
denotes a sum over all \emph{combinations} of positive integers $(k_1,\ldots,k_{m'})$ such that $k_1 + \ldots k_{m'} = m$ and
$$
\sum\limits_{i_1, \ldots i_{m'}} x_{i_1}^{k_1} \cdots x_{i_{m'}}^{k_{m'}}
$$
denotes a sum over all different monomials with $(k_1,\ldots,k_{m'})$ as exponents.

\subsection{Sensitivity for the Poisson regression}

We next prove the sensitivity result for the Poisson regression given in the main text.

\begin{lemma}
\label{lemma:Poisson_reg_sensitivity}
Let $\sigma_1,\sigma_2,\sigma_3,\sigma_4>0$.
Suppose $||\bx||_2 \leq R_x$ and $y \leq R_y$.
Consider the mechanism 
$$
\mathcal{M}(\bx) = \begin{bsmallmatrix} t_1(\bx) \\ t_2(\bx) \\ y t_1(\bx) \\ y t_2(\bx)\end{bsmallmatrix} + \mathcal{N}\left(0, \begin{bsmallmatrix} \sigma_1^2 I_d & 0 & 0 & 0 \\ 0 & \sigma_2^2 I_{d_2} & 0 & 0 \\
0 & 0 & \sigma_3^2 I_{d_1} & 0 \\
0 & 0 & 0 & \sigma_4^2 I_{d_2} \end{bsmallmatrix} \right).
$$
Then, the tight $(\epsilon,\delta)$-DP for $\mathcal{M}$ 
is obtained by considering a Gaussian mechanism with noise variance $\sigma_1^2$
and sensitivity
$$
\Delta = \frac{ 2 (c_2 + c_4) R_x^2 + c_3 + 1 }{ \sqrt{2(c_2+c_4)} },
$$
where $c_2 = \tfrac{\sigma_1^2}{\sigma_2^2}$,
$c_3 = \tfrac{\sigma_1^2}{\sigma_3^2} R_y^2$ and
$c_4 = \tfrac{\sigma_1^2}{\sigma_4^2} R_y^2$.
\begin{proof}
\begin{equation*}
     \begin{aligned}
     \mathcal{M}(\bx) & \sim \begin{bmatrix} t_1(\bx) \\ t_2(\bx) \\ y t_1(\bx) \\ y t_2(\bx)\end{bmatrix} + \mathcal{N}\left(0, \begin{bmatrix} \sigma_1^2 I_d & 0 & 0 & 0 \\ 0 & \sigma_2^2 I_{d_2} & 0 & 0 \\
0 & 0 & \sigma_3^2 I_{d_1} & 0 \\
0 & 0 & 0 & \sigma_4^2 I_{d_2} \end{bmatrix} \right) \\
   &  \sim \begin{bmatrix} I_d & 0 & 0 & 0\\
  0 &  \frac{\sigma_2}{\sigma_1} I_{d_2} & 0 & 0 \\
  0 & 0 & \frac{\sigma_3}{\sigma_1} I_d & 0 \\
   0 & 0 & 0 & \frac{\sigma_4}{\sigma_1} I_{d_2} \\
  \end{bmatrix} \left( \begin{bmatrix} t_1(\bx) \\ \frac{\sigma_1}{\sigma_2} t_2(\bx) \\ \frac{\sigma_1}{\sigma_3} y t_1(\bx) \\ \frac{\sigma_1}{\sigma_4} y t_2(\bx)\end{bmatrix} +  \mathcal{N}(0, \sigma_1^2 I ) \right).
     \end{aligned}
 \end{equation*}
Removing the constant scaling, we see that it is equivalent to consider the mechanism
\begin{equation*}
     \mathcal{M}(\bx) = \begin{bmatrix} t_1(\bx) \\ \frac{\sigma_2}{\sigma_1} t_2(\bx) \\ \frac{\sigma_1}{\sigma_3} y t_1(\bx) \\ \frac{\sigma_1}{\sigma_4} y t_2(\bx)\end{bmatrix} +  \mathcal{N}(0, \sigma_1^2 I ).
 \end{equation*}
Let $\bx,\bx' \in \mathbb{R}^d$ such that $||\bx||_2 \leq R_x$ and $||\bx'||_2 \leq R_x$.
Define
$$
F(\bx) =  \begin{bmatrix} t_1(\bx) \\ \frac{\sigma_1}{\sigma_2} t_2(\bx) \\ \frac{\sigma_1}{\sigma_3} y t_1(\bx) \\ \frac{\sigma_1}{\sigma_4} y t_2(\bx)\end{bmatrix}.
$$
Since
\begin{equation*} 
 \begin{aligned}
     || t_1(\bx) - t_1(\bx') ||^{2}_2 &= ||\bx||^2 + ||\bx'||^2 - 2 \langle \bx, \bx' \rangle \leq 2 R_x^2 - 2 \langle \bx, \bx' \rangle, \\
  || t_2(\bx) - t_2(\bx') ||^{2}_2 &= ||\bx||^4 + ||\bx'||^4 - 2 \langle \bx, \bx' \rangle^2  \leq 2R_x^4 - 2 \langle \bx, \bx' \rangle^2,
 \end{aligned}
 \end{equation*} 
 we have that 
\begin{equation*} 
 \begin{aligned}
\norm{F(\bx) - F(\bx')}^2 & = \norm{ t_1(\bx) - t_1(\bx') }^2 
+ \frac{\sigma_1^2}{\sigma_2^2} \norm{ t_2(\bx) - t_2(\bx') }^2 
+ \frac{\sigma_1^2}{\sigma_3^2} \norm{ y t_1(\bx) - y t_1(\bx') }^2 
+ \frac{\sigma_1^2}{\sigma_4^2} \norm{ y t_2(\bx) - y t_2(\bx') }^2 \\
& \leq - 2t + 2R_x^2 - 2 c_2 t^2 + 2c_2 R_x^4 -2 c_3 t + 2 c_3 R_x^2  - 2 c_4 t^2 + 2c_4 R_x^4,
 \end{aligned}
 \end{equation*} 
 where $c_2 = \frac{\sigma_1^2}{\sigma_2^2}$, $c_3 = \frac{\sigma_1^2}{\sigma_3^2} R_y^2 $, $c_4 = \frac{\sigma_1^2}{\sigma_4^2} R_y^2$ and $t = \langle \bx, \bx' \big\rangle$. The right hand side of this inequality has its maximum at
 $$
 t = - \frac{1+c_3}{2(c_2 + c_4)}
 $$
 which shows that
 $$
 \norm{F(\bx) - F(\bx')}^2 \leq \frac{ (2 (c_2+c_4) R_x^2 + c_3 + 1)^2  }{2 (c_2+c_4) }.
 $$
 This gives sensitivity bound of the claim.
\end{proof}
\end{lemma}

\section{Preliminary experiments on Poisson regression}
\para{Implementation.} We used Metropolis-Hastings algorithm to infer the model parameter posteriors for Poisson regression. We gave the regression coefficients $\btheta_{\bs}$ a standard normal prior and the data covariance $\bSigma_{\bs}$ an Inverse Wishart prior. As mentioned in the main draft, to the best of our knowledge, this is the first work that analyzes Poisson regression under DP constraints. Instead of employing the Isserlis theorem, we approximated the normal approximation parameters using MC integration. Specifying a more accurate, efficient, and scalable model for Poisson regression in sophisticated probabilistic programmings frameworks such as Stan is marked as a future exercise. We use synthetic data of $500$ samples generated with $\btheta=[0.3,-0.6,0.8]$ and a valid non-identity co-variance matrix. We filter out $||\bx||_2 > R_x=1$. The proposal standard deviation for the MH sampler was set to $0.01$. Our sampler runs for 50,000 iterations, out of which we discard the first 25,000 burn-in samples. We repeat each inference for $5$ times. 

\para{Results.} Figure~\ref{fig:cdfepsilon} compares private and non-private empirical CDFs for $\btheta'$s for various $\epsilon$ values within range $[0.1,1.1]$. The last plot shows the Kolmogorov-Smirnov scores between these CDFs for a few $\epsilon$ values in the same range. We note that the private CDFs tend to (partially) overlap on their non-private variants as $\epsilon$ increases. We suspect that the overlap is not as strong as it is in logistic regression due to a) more noise, which is a consequence of significantly more number of approximate sufficient statistics and larger range of $y$ (already explained by Lemma~\ref{lemma:Poisson_reg_sensitivity}), b) smaller sample size of $500$ and smaller number of inference repeats causing more uncertainty.  

\par We believe that it may be possible to improve these results with faster converging sampling algorithms and by designing better prior distributions. However, these preliminary explorations do demonstrate the merit of our model and tight sensitivity results.

\begin{center}
\begin{figure*}[h]
\includegraphics[scale=0.6]{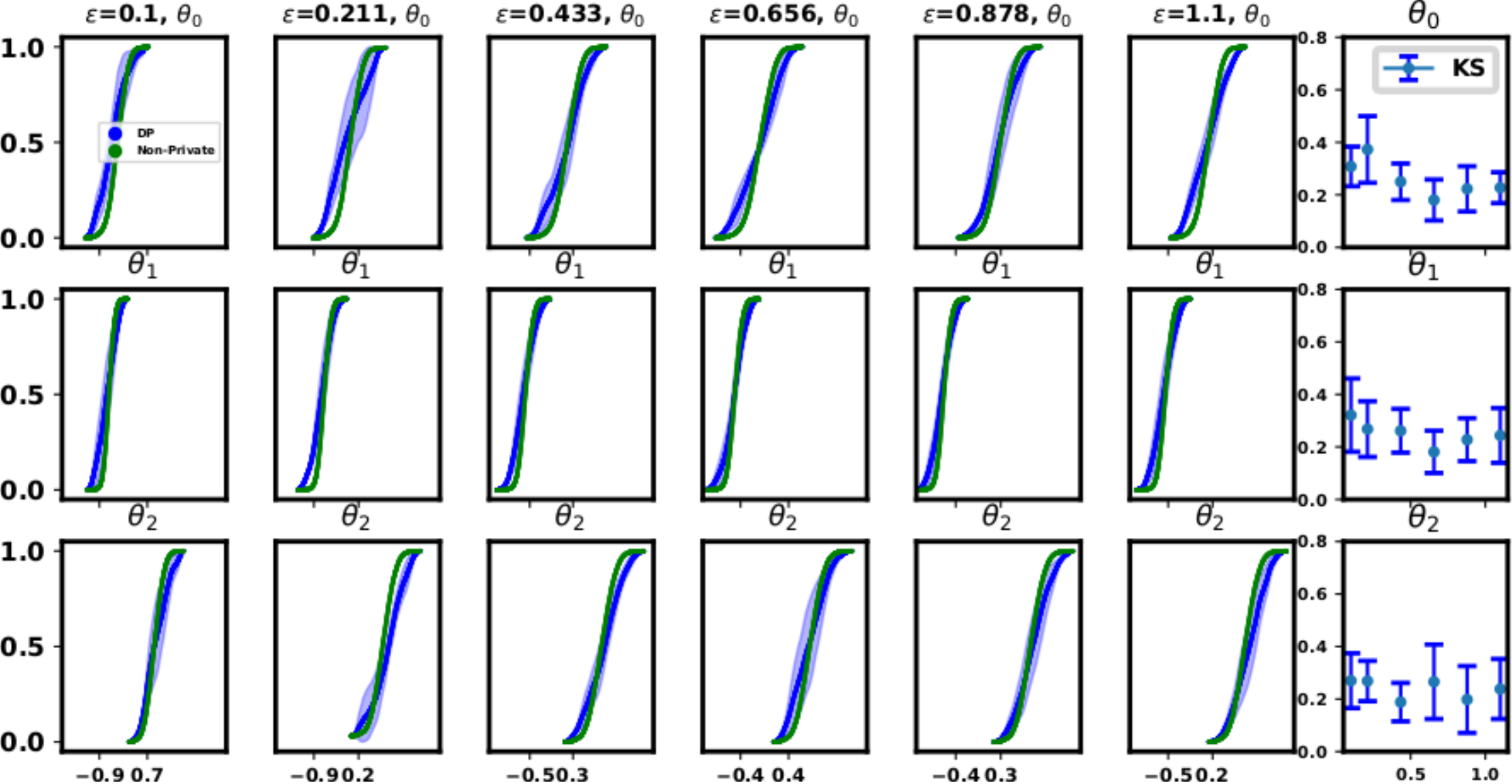}
\caption{Comparison of differentially private and non-private empirical CDFs for $\btheta$'s posteriors for Poisson regression for various $\epsilon$ values. We use a synthetic dataset of $N=500$ samples and $R_x=1,R_y=5,\delta=10^{-5}$. The right-most column shows the Kolmogorov-Smirnov scores between non-private and private empirical CDFs for the same set of $\epsilon$ values.}
    \label{fig:cdfepsilon}
\end{figure*}
\end{center}

\end{document}